%% file: main.tex
\documentclass{article} 
\usepackage{iclr2027_conference,times}
\usepackage{graphicx}
\usepackage{subcaption}
\usepackage{ifthen}
\usepackage{booktabs}
\usepackage{algorithm}
\usepackage{algpseudocode}
\usepackage{placeins}

\input{math_commands.tex}

\usepackage{amsthm}
\usepackage{amssymb}
\input{header.tex}

\input{cmd.tex}
\usepackage{hyperref}
\hypersetup{
    hidelinks,
    pdftitle={Dataset Pruning from First Principles: A Label-Free Linear Programming Approach},
    pdfauthor={Rodrigo Schuller and Francisco Ganacim},
    pdfsubject={},
    pdfkeywords={}
}
\usepackage{url}

\usepackage{microtype}          
\iclrfinalcopy
\showcommentsfalse
\markeditsfalse

\usepackage{etoolbox}
\makeatletter
\patchcmd{\@maketitle}
  {\lhead{Published as a conference paper at ICLR 2027}}
  {\lhead{}}
  {}{\PackageError{arxiv}{Could not remove the ICLR publication header}{Check the conference style.}}
\makeatother

\title{Dataset Pruning from First Principles: \\ A Label-Free Linear Programming Approach}

\author{Rodrigo Schuller \& Francisco Ganacim \\
Instituto de Matem\'atica Pura e Aplicada (IMPA) \\
Rio de Janeiro, Brazil \\
\texttt{\{rodrigo.loro,ganacim\}@impa.br}
}

\begin{document}

\maketitle

\begin{abstract}
Dataset pruning reduces a large training set to a representative subset while preserving model performance. Existing geometry-based methods typically assume that nearby points in embedding space share similar properties. Rather than imposing this assumption, we derive geometric selection criteria by reformulating unbiased subset selection as a variance minimization problem. Unbiasedness ensures that unweighted subset averages recover full-dataset averages in expectation, including losses and gradients at fixed model parameters. Specifically, we characterize a family of unbiased subset selection algorithms as a high-dimensional polytope. In this context, minimizing the expected sampling variance is a linear objective. Differences in sampling variance, averaged over rigid motions, admit closed-form pairwise expressions. Because the polytope has high dimension, directly applying standard linear programming is impractical. We instead use these expressions to construct an efficient vertex walk that optimizes an approximation of the variance objective while preserving unbiasedness, yielding a method that requires neither labels nor model training during selection. Across CIFAR-10, MNIST, and CelebA benchmarks, our method matches or exceeds uniform sampling in mean test accuracy at every evaluated budget and outperforms competing geometric methods in several settings, particularly at small selection budgets. Beyond dataset pruning, the same framework reduces stochastic-gradient variance by increasing diversity within mini-batches while keeping the batch size unchanged.
\end{abstract}

\section{Introduction}
Training large deep learning models on abundant data raises a basic question: given a fixed budget, which samples should we train on? The budget may be compute, as in self-supervised pretraining, where scaling laws favor larger models and datasets \citep{kaplan2020,hoffmann2022} at substantial computational cost \citep{strubell2019}, and where a well-chosen subset can outperform a larger random one \citep{sorscher2022}. It may also be annotation, as in remote sensing, where unlabeled archives are abundant but expert labels are expensive \citep{tuia2011}. Dataset pruning, also known as coreset selection, chooses a small, representative subset of the data. Addressing both cases requires selection that is label-free, so that annotation can be deferred until after selection, and training-free, so that selection does not itself consume the compute it is meant to save. We study methods that satisfy both requirements.

Without labels or training signals, selection must rely on the geometry of the embedding space alone. Existing geometry-based methods typically assume that nearby points in embedding space share similar properties \citep{moser2026}.
This assumption is a heuristic rather than a derived principle: it treats nearby points as redundant, which biases the distribution of the selected subset and can push accuracy below uniform random sampling \citep{guo2022}.
We drop this assumption and instead derive a geometric selection criterion from the \emph{principle of no preferred orientation}.

We seek a subset whose unweighted averages are unbiased, minimum-variance estimates of full-dataset averages of an unknown downstream function, such as a loss or gradient.  Since we know nothing about how downstream functions are aligned with the data, we assume no preferred orientation. Under this assumption, we average the mean squared estimation error over all rigid motions of the function, or equivalently of the data (Figure~\ref{fig:overview}, left). For unbiased selections, this averaged error is exactly the expected sampling variance.  On toroidal and spherical embedding spaces, this averaging yields closed-form expressions previously studied in the context of blue-noise sampling for computer graphics \citep{pilleboue2015}. We then use these formulas to build a novel unbiased coreset selection method.

These two geometries cover most embedding spaces used in practice. Many encoders are trained with cosine similarity and output $\ell_2$-normalized features, so their embeddings lie on the unit sphere; this is the case for contrastive image encoders \citep{chen2020,wang2020}, sentence embeddings for language \citep{reimers2019}, and multimodal models such as CLIP \citep{radford2021}. Any other embedding distribution with bounded support, or with rapidly decaying tails such as the approximately Gaussian latent codes of a variational autoencoder \citep{kingma2014}, can be enclosed in a sufficiently large box. Identifying opposite faces of this box places the data on a flat torus while leaving the pairwise distances between data points unchanged.

\newsavebox{\figonebox}
\begin{figure*}[t]
\centering
\sbox{\figonebox}{\includegraphics[width=0.776\textwidth]{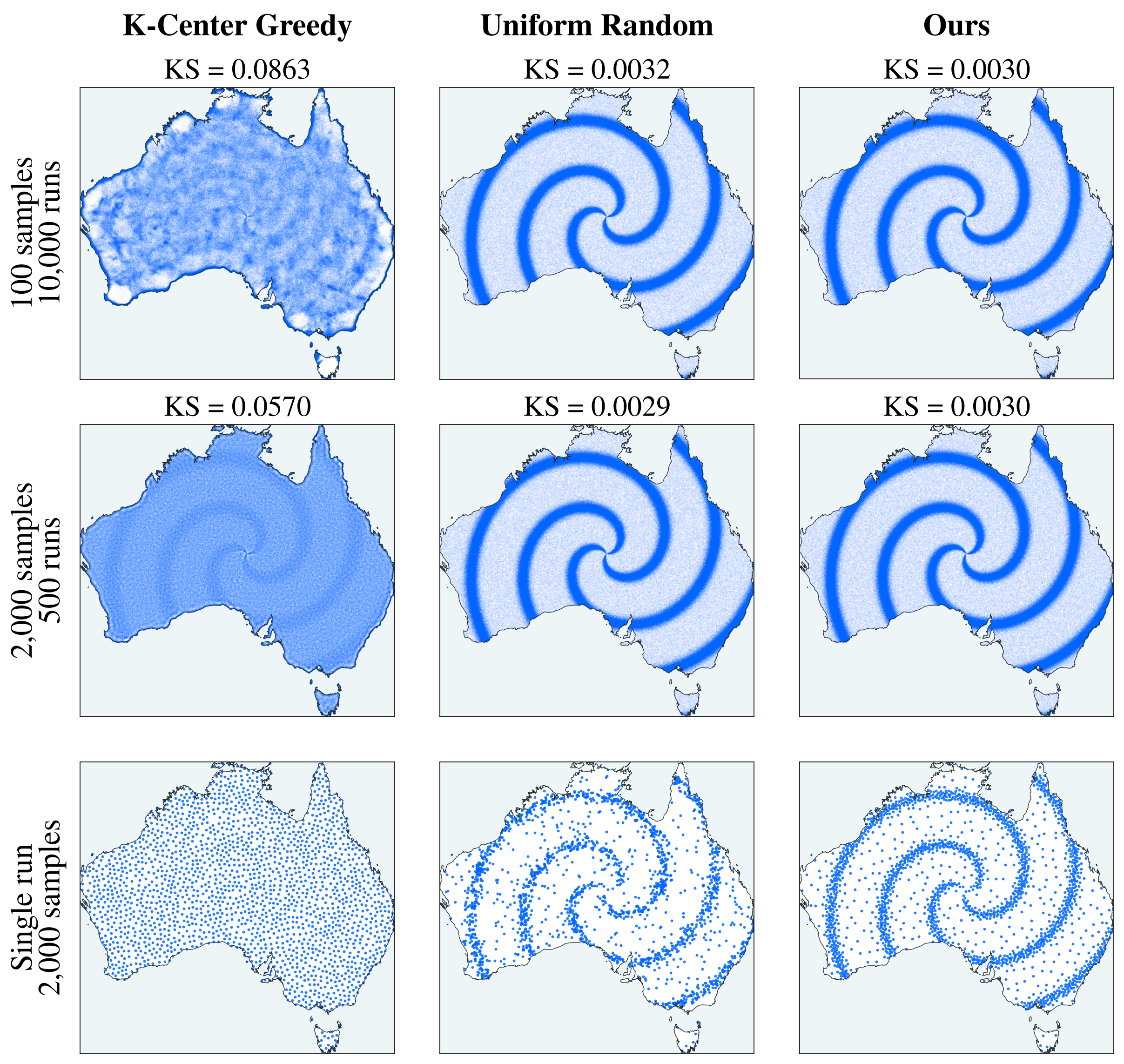}}
\includegraphics[height=\ht\figonebox]{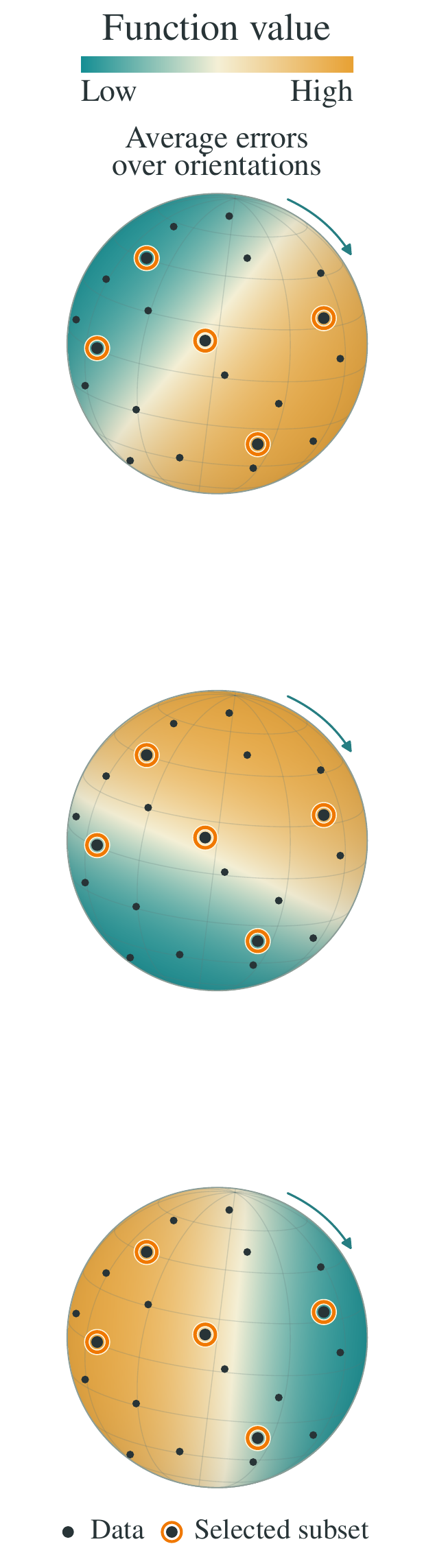}\hfill\usebox{\figonebox}
\caption{Left: orientation averaging. The data and the selected subset stay fixed while the function rotates, which is equivalent to rotating the data. The three orientations are illustrative; the estimation error is averaged over all of them. Right: spatial distributions of points selected from a vortex-shaped candidate distribution over Australia, shown in an equal-area projection. The top and middle rows pool 100 and 2,000 selected points from 10,000 and 500 runs, respectively, giving one million selected points per panel; the bottom row shows the first 2,000-point run. KS denotes the pooled two-dimensional Kolmogorov--Smirnov discrepancy against a vortex reference sample (lower is better; not a p-value).}
\label{fig:overview}
\end{figure*}

Our method rests on one reformulation: we treat a selection algorithm as a probability distribution over subsets of the training set (Section~\ref{sec:formulation}). Unbiased algorithms that select $k$ elements then form a convex polytope, and the orientation-averaged estimation error is a linear function over it, so unbiased minimum-variance selection is a linear program. The program has one coordinate per subset, far too many for standard solvers, so we instead develop an efficient vertex walk that decreases a variance surrogate while staying on the polytope, and hence unbiased.

{\color{magenta}We evaluate on CIFAR-10 and MNIST within the widely used DeepCore benchmarking framework, whose comparative study found uniform random selection to remain a strong baseline that competing methods did not consistently outperform} \citep{guo2022,moser2026}{\color{magenta}. In our experiments, our method matches or exceeds uniform selection in mean accuracy at every evaluated budget and outperforms the compared geometric methods at small budgets on CIFAR-10 and MNIST} (Section~\ref{sec:experiments}). Beyond dataset pruning, the same framework applies to mini-batch construction for stochastic gradient descent, where optimized batches show lower mean gradient variance than random batches at the same batch size (Section~\ref{sec:minibatches}).

\subsection{Contributions}

\begin{enumerate}
\item \textbf{Formulation.} We derive unbiased selection criteria from the principle of no preferred orientation and obtain closed-form pairwise variance expressions on the sphere and the torus (Section~\ref{sec:geometry}).
\item \textbf{Algorithm.} We give a label-free, training-free vertex walk that lowers a variance surrogate while preserving unbiasedness (Section~\ref{sec:formulation}).
\item \textbf{Experiments.} On CIFAR-10, MNIST, and CelebA, our method matches or exceeds uniform sampling at every budget and outperforms the compared geometric methods at small budgets on CIFAR-10 and MNIST (Section~\ref{sec:experiments}).
\item \textbf{Mini-batches.} The same framework reduces gradient variance in mini-batch construction at fixed batch size (Section~\ref{sec:minibatches}).
\end{enumerate}

\section{Related work}
\label{sec:related}

We refer to \citet{moser2026} for a recent taxonomy that organizes coreset selection into training-free, training-oriented, and label-free approaches. The DeepCore benchmark \citep{guo2022} compares many of these methods on CIFAR-10 and ImageNet and finds that uniform random selection remains a strong baseline that many sophisticated methods fail to beat consistently. Uniform sampling is unbiased; we keep this property by construction and search for lower variance within it.

Training-oriented methods score points with signals collected during training: gradient norms early in training \citep{paul2021}, gradient matching \citep{mirzasoleiman2020,killamsetty2021grad}, and bilevel formulations \citep{killamsetty2021glister}.  These methods can reach higher accuracy than training-free selection, but they require labels and at least one training pass over the full dataset, which is precisely the cost that training-free and label-free selection is designed to avoid.

Closest to our setting are selection rules that operate directly on an embedding of the data, without labels or training runs. Herding \citep{welling2009,chen2010} greedily selects points so that the subset mean tracks the full-data mean; $k$-center greedy \citep{sener2018} minimizes a covering radius over the feature space; and \citet{sorscher2022} prune by distance to self-supervised cluster prototypes. Related geometric criteria additionally exploit labels or training signals: moderate coresets keep points at intermediate distances from class centers \citep{xia2023}, and coverage-centric selection stratifies the budget over difficulty scores to preserve coverage \citep{zheng2023}. All of these methods encode some version of the heuristic that nearby points are redundant, and all of them bias the distribution of the selected subset, which can drive accuracy below uniform sampling, particularly at small budgets and under class imbalance \citep{guo2022,zheng2023}. In contrast, we derive the geometric intuition from first principles and search only among selections that are provably unbiased.

Unbiasedness is a classical requirement in the subset-sampling literature. Sensitivity-based coresets attach importance weights to the selected points so that the weighted subset loss is an unbiased estimator of the full loss \citep{feldman2011,bachem2017}\purple{, and determinantal point processes yield Monte Carlo estimators with faster-than-standard convergence rates} \citep{bardenet2020}. Kernel thinning \citep{dwivedi2024} recursively splits the data into balanced blocks using a kernel-dependent randomized procedure, then selects a block with small maximum mean discrepancy and refines it greedily. The last two steps can bias the selection. We make the opposite choice: optimize the partition, then return a uniformly chosen block. \purple{Unlike the weighted estimators above, our selections are unweighted, and unbiasedness enters as a linear constraint that turns the set of selection algorithms into a convex polytope.}

\green{Mini-batch construction is a second setting where unbiased subset averages matter, since stochastic gradient descent relies on uniformly sampled batches being unbiased. Importance sampling draws examples non-uniformly and reweights their gradients to stay unbiased} \citep{needell2014,zhao2015,katharopoulos2018}\green{, at the cost of per-example gradient information. Stratified sampling draws each batch proportionally from clusters of similar examples} \citep{zhao2014}\green{, and repulsive point processes make similar examples less likely to share a batch, }whether determinantal \citep{zhang2017} or Poisson disk \citep{zhang2019}, at the cost of a batch gradient that is, in general, unbiased only for a re-weighted objective. \green{Example ordering is a related lever: random reshuffling improves on sampling with replacement} \citep{mishchenko2020}\green{, and GraB uses herding on stale gradients to order examples so that partial sums of gradients track the full gradient} \citep{lu2022}\green{. Our batches are the blocks of an unbiased partition recomputed at each epoch from the embedding geometry alone.} Like stratified and repulsive sampling, this needs no gradient information, but the partition minimizes our variance surrogate rather than a fixed clustering or a repulsion heuristic, and {\color{magenta}a uniformly chosen batch has an unbiased gradient for the original objective at fixed model parameters}. Like reshuffling, every example is used once per epoch.

Finally, our closed forms are related to the variance analysis of sampling patterns in computer graphics, where the error of Monte Carlo integration under randomly transformed point sets admits spectral expressions on the torus and on the sphere \citep{pilleboue2015,oztireli2016,singh2019}. The Poisson disk batches of \citet{zhang2019} come from the same machinery: the variance of the stochastic gradient is expressed through the first- and second-order product densities of a point process over the data. These works analyze continuous point processes; we transport the machinery to the discrete setting of selecting subsets of a fixed dataset, average the estimation error over orientations of the unknown function, and couple the resulting closed forms with the polytope characterization of unbiased selections to obtain a practical algorithm that is unbiased by construction.

\section{Problem formulation}
\label{sec:formulation}
 
\subsection{Dataset pruning}
Let $\mathcal{T} = \{(x_i, y_i)\}_{i=1}^N$ be a training dataset, consisting of $N$ independent samples from the joint probability distribution $P(X,Y)$. We define $x_i \in \mathcal{X}$ as the inputs and $y_i \in \mathcal{Y}$ as the ground-truth labels. Dataset pruning aims to find a subset $\mathcal{S} \subset \mathcal{T}$ such that a model $M_{\theta_\mathcal{S}} \in \mathcal{M}$  trained on $\mathcal{S}$ and with parameters $\theta_{\mathcal{S}}$ achieves a good generalization error \citep{moser2026}:
$$
    \mathcal{S}^*
:=
    \argmin_{\substack{
        \mathcal{S} \subset \mathcal{T},
        \frac{\abs{\mathcal{S}}}{\abs{\mathcal{T}}} \leq 1 - \alpha
        }
    }
    \E_{(x,y) \sim P} \big[ \mathcal{L}(x,y; M_{\theta_\mathcal{S}}) \big],
$$ \purple{where $\mathcal{S}^*$ is the optimal subset,} $\alpha \in (0,1)$ is known as the pruning ratio\purple{,} and $\mathcal{L}$ is the loss function. In practical applications, a subset $\mathcal{S} \subset \mathcal{T}$ is considered good if $\abs{\mathcal{S}} \ll \abs{\mathcal{T}}$ and $\E_{(x,y) \sim P} \big[ \mathcal{L}(x,y; M_{\theta_\mathcal{S}}) \big] \approx \E_{(x,y) \sim P} \big[ \mathcal{L}(x,y; M_{\theta_\mathcal{T}}) \big]$, \purple{that is, if the model trained on the subset generalizes nearly as well as the model trained on the full dataset.}

\subsection{Unbiased selection}
\purple{We want selection algorithms that are unbiased, meaning that subset averages equal full-dataset averages in expectation. Under this constraint, the mean squared error of any quantity we estimate from the subset, whether a loss, a gradient, or an accuracy, reduces to its variance, which is therefore the quantity to minimize. To make this precise, we model} a selection algorithm $p$ as a discrete probability measure on the subsets of $\mathcal{T}$\purple{. The set of all such measures is
$
    \mathcal{P}
:=
    \bigl\{
        q \in [0,1]^{2^{\abs{\mathcal{T}}}} \,:\, \sum_{S \subseteq \mathcal{T}} q_S = 1
    \bigr\},
$
a probability simplex with $2^{\abs{\mathcal{T}}}$ coordinates. If $p \in \mathcal{P}$, then $p_S$ is the probability that $p$ assigns to subset $S$. Unbiasedness is then a condition on $p$:}

\begin{definition}
\label{def:unbiased}
\purple{A selection algorithm $p \in \mathcal{P}$ is \emph{unbiased} if} for every function $f : X \to V$ from an arbitrary set $X \supset \mathcal{T}$ to a vector space $V$ over $\R$,
$$
    \E_{S \sim p}\left[
        \frac{1}{\abs{S}}
        \sum_{x \in S}
        f(x)
    \right]
=
    \frac{1}{\abs{\mathcal{T}}}
    \sum_{t \in \mathcal{T}} f(t).
$$
\end{definition}

\purple{Two restricted classes of unbiased algorithms will be central to what follows.}

\begin{definition}
\label{def:k-selection}
\purple{An unbiased selection algorithm $p^{(k)} \in \mathcal{P}$ \emph{selects $k$ elements} if
$
    \forall S \subseteq \mathcal{T}, \; p^{(k)}_S > 0 \implies \abs{S} = k.
$
The set of all such algorithms is denoted by $\mathcal{U}_k$.}
\end{definition}

\begin{definition}
\label{def:permutation}
\purple{An algorithm $p^{(k)} \in \mathcal{U}_k$ is a \emph{uniform $k$-selection} algorithm, or simply a \emph{$k$-selection}, if
$
    \forall S \subseteq \mathcal{T},\;
        p^{(k)}_S
    \in
        \bigl\{0, k/\abs{\mathcal{T}}\bigr\}.
$}
\end{definition}

\purple{It follows that a $k$-selection exists only if $k$ divides $\abs{\mathcal{T}}$, and that every $k$-selection amounts to breaking $\mathcal{T}$ into $\abs{\mathcal{T}} / k$ blocks of size $k$ and selecting one of them uniformly at random.}

\section{From orientation averages to practical objectives}
\label{sec:geometry}

This section shows how averaging over all orientations yields a practical objective. Since nothing is known about how the downstream function is aligned with the data, no orientation is preferred, and the variance is averaged over all orientations of the function, or equivalently of the data (Figure~\ref{fig:overview}, left). On the sphere, a change of orientation is a rigid motion; on the flat torus, it is a rotation followed by a translation, with points wrapped back onto the torus. Either way the domain is mapped to itself, so the average never involves values of the function on unrelated regions outside it. {\color{magenta}Averaging over these transformations} turns differences in variance between selection algorithms into sums over pairs of selected points (Theorems~\ref{theorem:torus} and~\ref{theorem:sphere}, proved in the Appendix), and these pairwise expressions motivate surrogate objectives that depend only on the geometry of the selection.

Let $f \in L^2(\mathcal{X} \to \R)$ be a function whose full-dataset average is to be estimated by subset averages. For an unbiased selection algorithm $p \in \mathcal{P}$, the variance coincides with the mean squared error and is defined by
$$
    V_p
:=
    \E_{S \sim p}
    \Bigg[
    \bigg(
        \frac{1}{\abs{\mathcal{T}}} \sum_{t \in \mathcal{T}} f(t)
    -
        \frac{1}{\abs{S}} \sum_{x \in S} f(x)
    \bigg )^2
    \Bigg].
$$
After a change of orientation $g$ of the points, the variance becomes 
$$
    V_p(g)
:=
    \E_{S \sim p}
    \Bigg[
    \bigg(
        \frac{1}{\abs{\mathcal{T}}} \sum_{t \in \mathcal{T}} (f \circ g)(t)
    -
        \frac{1}{\abs{S}} \sum_{x \in S} (f \circ g)(x)
    \bigg)^2
    \Bigg].
$$
By Theorems~\ref{theorem:torus} and~\ref{theorem:sphere}, for two unbiased selection algorithms $p, q \in \mathcal{P}$, averaging the difference $V_p(g) - V_q(g)$ over the uniform distribution $\mathcal{G}$ of orientations on the torus or the sphere gives
\begin{align*}
    \E_{g \sim \mathcal{G}}
    \big[ V_p(g) - V_q(g) \big]
=
    \E_{S \sim p} \bigg[
        \frac{1}{\abs{S}^2}
        \sum_{x,y \in S}
        v_{\mathcal{X}}\big( h_{\mathcal{X}}(x, y) \big)
    \bigg]
-
    \E_{S \sim q} \bigg[
        \frac{1}{\abs{S}^2}
        \sum_{x,y \in S}
        v_{\mathcal{X}}\big( h_{\mathcal{X}}(x, y) \big)
    \bigg],
\end{align*}
in which
\begin{align}
    \label{eq:torus-kernel}
    v_{\T^d}(r)
&:=
    \sum_{j \in \Z^d}
    \abs{\hat{f}(j)}^2
    \frac{\Gamma(d/2)}
         {\left(\pi r \norm{j}_2 \right)^{d/2 -1}}
    J_{d/2-1} (2 \pi r \norm{j}_2)
,\;
    h_{\T^d}(x, y) := \norm{x - y}_2 \\
    \label{eq:sphere-kernel}
    v_{S^d}(t)
&:=
    \sum_{\ell \in \N}
    \norm{f_\ell}_{L^2(\sigma)}^2
    P_\ell^{d+1}(t)
,\;
    h_{S^d}(x, y) := x \cdot y.
\end{align}
Here $\hat{f}(j)$ is the $j$-th Fourier coefficient of $f$, $J_\alpha$ is the Bessel function of the first kind, $\sigma$ is the normalized surface measure on $S^d$, with $\sigma(S^d)=1$, $f_\ell$ is the degree-$\ell$ spherical harmonic component of $f$, and $P_\ell^{d+1}$ is the generalized Legendre polynomial of dimension $d+1$ and degree $\ell$.

In both cases, the orientation-averaged variance difference is a linear combination of oscillating functions of the pairwise geometry (Figure~\ref{fig:geometric-factors}). The spectral weights $\lvert\hat{f}(j)\rvert^2$ and $\lVert f_\ell\rVert_{L^2(\sigma)}^2$ depend on the unknown function $f$, but the oscillating factors they multiply depend only on the distance between pairs of selected points on the torus, or on their inner product on the sphere. The envelopes of these factors identify regions of low correlation: large separations on the torus, and inner products near zero, corresponding to approximately perpendicular points, on the sphere. Placing pairs in these regions keeps every term of the combination small regardless of $f$, which motivates a variance surrogate built from a geometry-only pairwise score $\tilde{v}_{\mathcal{X}}$:
\begin{align}
    \widetilde V_{\mathcal{X}}(p) :=
    \E_{S \sim p} \left[
        \frac{1}{\abs{S}^2}
        \sum_{x,y \in S}
            \tilde{v}_{\mathcal{X}}(x, y)
    \right].
\end{align}
The difference $\widetilde V_{\mathcal{X}}(p)-\widetilde V_{\mathcal{X}}(q)$ then stands in for the orientation-averaged variance difference $\E_{g\sim\mathcal{G}}[V_p(g)-V_q(g)]$. The experiments use the inverse-distance score $\tilde{v}_{\mathcal{X}}(x, y) = \frac{1}{\norm{x - y} + 0.01}$, which favors large separations, following the torus motivation, although exact upper bounds could also be used. {\color{magenta}On the sphere, scores that favor distant points also favor approximately perpendicular pairs when pairwise inner products are nonnegative. We observed this positive-inner-product regime for almost all sampled pairs of IMDb embeddings.}

\begin{figure*}[t]
\centering
\includegraphics[width=.8\textwidth]{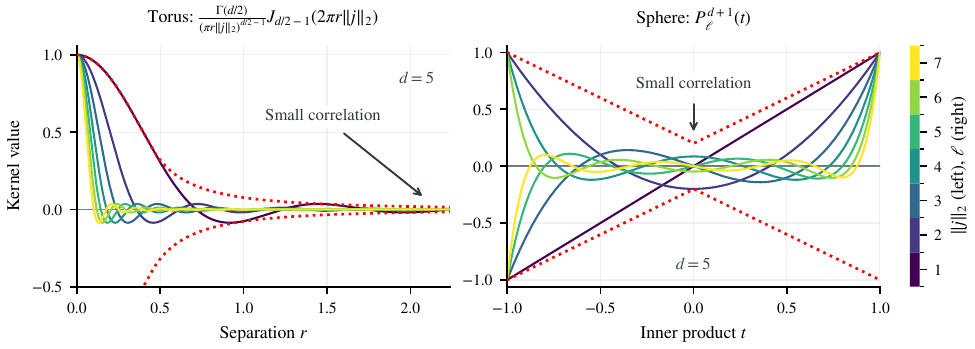}
\caption{Left: the normalized Bessel factor as a function of the pairwise separation $r$. Right: the normalized Gegenbauer polynomial $P_\ell^{d+1}(t)$ as a function of the inner product $t$, both for $d=5$. Colors identify the frequency norm $\lVert j\rVert_2$ (left) or the degree $\ell$ (right), from 1 to 7; the constant mode is omitted. Dotted red curves show the envelope of each family, and arrows mark the regions of low correlation that motivate the surrogate.}
\label{fig:geometric-factors}
\end{figure*}

\section{Our method}
\label{sec:method}

We partition $\mathcal{T}$ into $m=\abs{\mathcal{T}}/k$ blocks $S_1,\ldots,S_m$, each containing $k$ points.\footnote{We consider $k<\abs{\mathcal{T}}$. If $k$ does not divide $\abs{\mathcal{T}}$, we first discard $\abs{\mathcal{T}}\bmod k$ points uniformly at random and apply the procedure to the retained dataset. The resulting selection remains unbiased for the original dataset.} Selecting one of these blocks uniformly at random defines a $k$-selection $p$ (see Definition~\ref{def:permutation}), which corresponds to a vertex of $\mathcal{U}_k$. Swapping two points from different blocks preserves the partition structure and therefore defines another $k$-selection. Algorithm~\ref{alg:partition-optimization} exploits this property by repeatedly proposing swaps between blocks and accepting them only when they decrease the variance surrogate from Section~\ref{sec:geometry}:
\begin{equation}
    \widetilde V_{\mathcal{X}}(p)
    = \frac{1}{mk^2}
    \sum_{i=1}^{m}\sum_{x,y\in S_i}\tilde{v}_{\mathcal{X}}(x,y).
\end{equation}

Lemma~\ref{lem:unbiased-selection-polytope} shows that each
\purple{$k$-selection} of the training set corresponds to a vertex of the
polytope of unbiased selection algorithms with selection budget $k$. It also
shows that the expected sum of pairwise kernel evaluations is a linear
objective in the subset-selection probabilities. Thus, our swap-based method moves among extreme points of
the feasible polytope, each of which is optimal for some linear objective.

\begin{algorithm}[t]
\caption{Swap-based partition optimization}
\label{alg:partition-optimization}
\algrenewcommand{\algorithmicensure}{\textbf{Output:}}
\begin{algorithmic}[1]
\Require Dataset $\mathcal{T}$, subset size $k$ with $k\mid\abs{\mathcal{T}}$,
    variance surrogate $\tilde{v}_{\mathcal{X}}$, number of iterations $L$
\Ensure A subset of $k$ points
\State $m \gets \abs{\mathcal{T}}/k$
\State Precompute $V_{x,y} \gets \tilde{v}_{\mathcal{X}}(x,y)$ for all $x,y\in\mathcal{T}$
\State Partition $\mathcal{T}$ into blocks $S_1,\ldots,S_m$, each stored as a list of $k$ points
\For{$\ell=1,\ldots,L$}
    \State Sample $i,j$ uniformly without replacement from $\{1,\ldots,m\}$
    \State Sample $r,s$ independently and uniformly from $\{1,\ldots,k\}$
    \State $\Delta \gets -\sum_{q=1}^{k}\bigl(V_{S_i[r],S_i[q]}+V_{S_j[s],S_j[q]}\bigr)$
    \State Swap $S_i[r]$ and $S_j[s]$
    \State $\Delta \gets \Delta+\sum_{q=1}^{k}\bigl(V_{S_i[r],S_i[q]}+V_{S_j[s],S_j[q]}\bigr)$
    \If{$\Delta\geq 0$}
        \State Swap $S_i[r]$ and $S_j[s]$ back
    \EndIf
\EndFor
\State Sample $i$ uniformly from $\{1,\ldots,m\}$
\State \Return $S_i$
\end{algorithmic}
\end{algorithm}

As a faster alternative to iterative swaps, we encode each embedding as an $N_b$-bit vector whose scaled Hamming distances approximate the original $\ell_1$ distances. We maintain per-bit counts for each bucket, allowing us to compute a point's total Hamming distance to its members without storing pairwise distances. Using these distances as a surrogate, we greedily assign points to distant buckets while ensuring that each bucket ends with exactly $k$ points. Selecting one completed bucket uniformly at random then preserves unbiasedness and yields a partition-induced vertex of $\mathcal{U}_k$. Excluding encoding, assignment costs $O(mN_b\abs{\mathcal{T}})$, which is linear in dataset size for a fixed selection fraction and number of bits. This greedy alternative gives up some fidelity to our derived variance surrogate, and may perform fewer comparisons, in exchange for faster execution (see Section~\ref{sec:experiments}). 

\section{Experiments}
\label{sec:experiments}

\subsection{Dataset pruning}
We evaluate our method on CIFAR-10, MNIST, CelebA \citep{liu2015}, and IMDb \citep{maas2011} at selection fractions from 0.005 to 0.3 of the training set. We compare against three training-free, label-free baselines: uniform sampling, herding \citep{welling2009,chen2010}, and $k$-center greedy \citep{sener2018}. We use the DeepCore framework \citep{guo2022} without class-balanced selection quotas, since enforcing them requires labels. Herding uses a corrected implementation, since the original DeepCore implementation contains a bug. Tables report mean test accuracy in percent with 95\% Wald confidence-interval half-widths; boldface marks the highest mean in each column, not a statistically significant difference.

On CIFAR-10 and MNIST, {\color{magenta}the embedding-based methods use features from a ResNet18 trained for 10 epochs on the full labeled training set, following DeepCore's protocol for benchmark comparability. A freshly initialized ResNet18 is then trained on the selected subset for 200 epochs.} Training uses stochastic gradient descent (SGD) with a batch size of 128 and a learning rate of 0.1. {\color{magenta}Each setting uses 30 jobs. Each job comprises five selection-and-training runs; we average their best-epoch test accuracies within each job and report means and confidence intervals across jobs.}

CelebA and IMDb use pretrained embeddings. On CelebA, we use the latent codes of a variational autoencoder (VAE) \citep{kingma2014}, whose regularized, approximately Gaussian embedding fits the toroidal setting. On IMDb, a film-review dataset, we use {\color{magenta}Qwen3-Embedding-8B} embeddings \citep{zhang2025} to evaluate the spherical and label-free setting. {\color{magenta}Among one million uniformly sampled pairs of distinct IMDb training embeddings, all inner products were positive and over 99\% exceeded 0.2} (Figure~\ref{fig:imdb-inner-products}){\color{magenta}. Thus, increasing pairwise separation favors movement toward orthogonality in the observed regime.} Test accuracy is that of a multilayer perceptron (MLP) trained on the selected subset, averaged over 100 independent runs per setting on CelebA and 200 on IMDb. Additional experimental details are given in Appendix~\ref{app:experimental-details}.

\begin{table}[t]
\centering
\caption{Test accuracy (\%) on CIFAR-10 and MNIST.}
\label{tab:cifar-mnist}
\small
\setlength{\tabcolsep}{3pt}
\begin{tabular*}{\linewidth}{@{\extracolsep{\fill}}lrrrrr@{}}
\toprule
\multicolumn{6}{c}{\textbf{CIFAR-10}} \\
\midrule
Method / selected fraction & 0.005 & 0.01 & 0.05 & 0.1 & 0.3 \\
\midrule
Herding & 24.65 $\pm$ 0.30 & 30.63 $\pm$ 0.39 & 56.66 $\pm$ 0.48 & 74.76 $\pm$ 0.34 & 90.61 $\pm$ 0.07 \\
Ours & 31.30 $\pm$ 0.25 & {\color{magenta}\textbf{38.30}} $\pm$ 0.28 & \textbf{63.84} $\pm$ 0.32 & \textbf{76.91} $\pm$ 0.26 & 90.59 $\pm$ 0.06  \\
Ours (fast) & {\color{magenta}\textbf{31.30} $\pm$ 0.28} & {\color{magenta}38.20 $\pm$ 0.28} & {\color{magenta}63.63 $\pm$ 0.41} & {\color{magenta}76.48 $\pm$ 0.21} & {\color{magenta}90.67 $\pm$ 0.05} \\
Uniform & 30.63 $\pm$ 0.54 & 38.20 $\pm$ 0.51 & 63.22 $\pm$ 0.69 & 76.35 $\pm$ 0.36 & 90.49 $\pm$ 0.11 \\
$k$-center greedy & 28.95 $\pm$ 0.27 & 34.93 $\pm$ 0.23 & 60.19 $\pm$ 0.51 & 75.73 $\pm$ 0.35 & \textbf{91.03} $\pm$ 0.07 \\
\midrule
\multicolumn{6}{c}{\textbf{MNIST}} \\
\midrule
Herding & 64.25 $\pm$ 1.13 & 82.83 $\pm$ 0.67 & 98.22 $\pm$ 0.05 & 99.17 $\pm$ 0.02 & 99.56 $\pm$ 0.01 \\
Ours & \textbf{88.44} $\pm$ 0.40 & \textbf{92.83} $\pm$ 0.25 & 98.34 $\pm$ 0.05 & 99.13 $\pm$ 0.02 & 99.53 $\pm$ 0.01 \\
Ours (fast) & 87.87 $\pm$ 0.45 & 92.57 $\pm$ 0.21 & 98.26 $\pm$ 0.04 & 99.12 $\pm$ 0.01 & 99.53 $\pm$ 0.01 \\
Uniform & 86.39 $\pm$ 0.78 & 92.41 $\pm$ 0.35 & 98.20 $\pm$ 0.07 & 99.09 $\pm$ 0.03 & 99.53 $\pm$ 0.01 \\
$k$-center greedy & 68.93 $\pm$ 1.28 & 86.21 $\pm$ 0.51 & \textbf{98.96} $\pm$ 0.02 & \textbf{99.49} $\pm$ 0.01 & \textbf{99.65} $\pm$ 0.00 \\
\bottomrule
\end{tabular*}
\end{table}

\begin{table}[t]
\centering
\caption{CelebA mean test attribute accuracy (\%).}
\label{tab:celeba}
\small
\setlength{\tabcolsep}{3pt}
\begin{tabular*}{\linewidth}{@{\extracolsep{\fill}}lrrrrr@{}}
\toprule
Method / selected fraction & 0.005 & 0.01 & 0.05 & 0.1 & 0.3 \\
\midrule
Herding & \textbf{82.89} $\pm$ 0.04 & 85.04 $\pm$ 0.02 & 88.56 $\pm$ 0.01 & 89.52 $\pm$ 0.01 & 90.62 $\pm$ 0.00 \\
Ours & 82.65 $\pm$ 0.07 & \textbf{85.08} $\pm$ 0.03 & \textbf{88.64} $\pm$ 0.01 & \textbf{89.60} $\pm$ 0.01 & \textbf{90.66} $\pm$ 0.01 \\
Ours (fast) & 82.67 $\pm$ 0.06 & 85.02 $\pm$ 0.03 & 88.63 $\pm$ 0.01 & 89.58 $\pm$ 0.01 & 90.64 $\pm$ 0.01 \\
Uniform & 82.57 $\pm$ 0.07 & 84.95 $\pm$ 0.03 & 88.58 $\pm$ 0.02 & 89.57 $\pm$ 0.01 & 90.65 $\pm$ 0.01 \\
$k$-center greedy & 80.18 $\pm$ 0.03 & 81.80 $\pm$ 0.07 & 88.27 $\pm$ 0.01 & 89.46 $\pm$ 0.01 & 90.62 $\pm$ 0.00 \\
\bottomrule
\end{tabular*}
\end{table}

\newsavebox{\imdbtablebox}
\newsavebox{\imdbcaptionbox}
\begin{table}[t]
\centering
\sbox{\imdbtablebox}{%
\begin{minipage}{0.58\linewidth}
\centering
\caption{IMDb test accuracy (\%).}
\label{tab:imdb}
\small
\setlength{\tabcolsep}{3pt}
\begin{tabular}{lrrr}
\toprule
Method / fraction & 0.005 & 0.01 & 0.05 \\
\midrule
Herding & \textbf{92.73} $\pm$ 0.01 & 93.46 $\pm$ 0.01 & 94.47 $\pm$ 0.02 \\
Ours & {\color{magenta}92.44 $\pm$ 0.09} & 93.53 $\pm$ {\color{magenta}0.05} & 94.50 $\pm$ {\color{magenta}0.02} \\
Ours (fast) & 92.52 $\pm$ 0.09 & \textbf{93.54} $\pm$ 0.06 & \textbf{94.53} $\pm$ 0.02 \\
Uniform & {\color{magenta}92.03 $\pm$ 0.12} & {\color{magenta}93.33 $\pm$ 0.06} & {\color{magenta}94.45 $\pm$ 0.02} \\
kCenterGreedy & {\color{magenta}84.76 $\pm$ 0.61} & {\color{magenta}91.03 $\pm$ 0.08} & {\color{magenta}94.29 $\pm$ 0.02} \\
\bottomrule
\end{tabular}
\end{minipage}%
}%
\sbox{\imdbcaptionbox}{%
\begin{minipage}{0.40\linewidth}
\captionsetup{type=figure,skip=2pt,justification=raggedright}
\caption{{\color{magenta}IMDb inner products.}}
\label{fig:imdb-inner-products}
\end{minipage}%
}%
\usebox{\imdbtablebox}\hfill
\raisebox{-\dp\imdbtablebox}{\vbox{%
\hbox to 0.40\linewidth{\hfil
\includegraphics[width=0.40\linewidth,
  totalheight=\dimexpr\ht\imdbtablebox+\dp\imdbtablebox
    -\ht\imdbcaptionbox-\dp\imdbcaptionbox\relax,
  keepaspectratio]{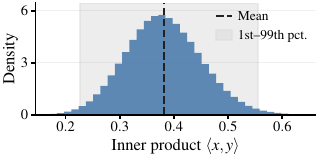}%
\hfil}%
\nointerlineskip
\hbox{\usebox{\imdbcaptionbox}}%
\kern0pt 
}}
\end{table}

On all four datasets we use the surrogate score $\tilde{v}_{\mathcal{X}}(x,y)=1/(\lVert x-y\rVert_2+0.01)$ and run Algorithm~\ref{alg:partition-optimization} for {\color{magenta}$65{,}536k$ proposed swaps, reduced to $2{,}048k$ only at fraction $0.3$ on CIFAR-10 and MNIST}. The fast variant, reported as Ours (fast), replaces Euclidean distances by Hamming distances between 64k-bit binary representations of the embeddings. Tables~\ref{tab:cifar-mnist}, \ref{tab:celeba}, and~\ref{tab:imdb} show that our method is particularly competitive at small budgets on CIFAR-10 and MNIST, whereas $k$-center greedy achieves the highest mean accuracy at larger budgets on these datasets. On CelebA and IMDb, our method has the highest mean accuracy at every budget except the smallest, where herding leads. Unlike the other methods, ours matches or exceeds uniform sampling in mean accuracy at every evaluated budget.

Figure~\ref{fig:benchmark-scaling} shows the recorded runtimes for 4096-dimensional embeddings. The largest run of our method, with $2^{22}$ embeddings, takes 153.1 seconds, whereas both baselines complete at most $2^{18}$ embeddings within the 300-second limit of the sweep.

\begin{figure}[t]
\centering
\includegraphics[width=\linewidth]{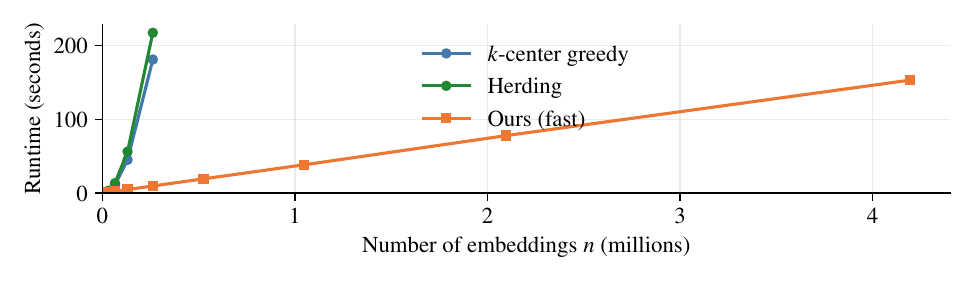}
\caption{Recorded runtime versus input size (linear axes).
We select $n/128$ points from standard-normal,
4096-dimensional float32 embeddings on an NVIDIA GeForce RTX 5090. We use binary representations with $64$k bits.}
\label{fig:benchmark-scaling}
\end{figure}

\subsection{Mini-batch variance reduction}
\label{sec:minibatches}

We evaluate mini-batch construction {\color{magenta}on CelebA's 40 attributes} using ResNet-18 trained with AdamW (initial learning rate $10^{-3}$, weight decay $10^{-4}$) and cosine annealing ($T_{\max}=30$, $\eta_{\min}=0$). At each epoch, we use the previous epoch's 512-dimensional pooled features, extracted in evaluation mode, as the selection features for batch optimization. {\color{magenta}Before each reported epoch, we compare random and optimized partitions of the same samples at fixed model parameters and batch size 128.} For each scheme, we estimate variance as the mean squared Euclidean distance of the mini-batch loss gradients from their mean. {\color{magenta}We report means $\pm$ 95\% Wald confidence-interval half-widths across 256 paired runs.} Table~\ref{tab:minibatch-variance} shows lower mean gradient variance for optimized batches than for random batches at every evaluated epoch.

\begin{table}[t]
\centering
\caption{Gradient variance of random and optimized mini-batches at fixed batch size.}
\label{tab:minibatch-variance}
\small
\setlength{\tabcolsep}{3pt}
\begin{tabular*}{\linewidth}{@{\extracolsep{\fill}}rrrr@{}}
\toprule
Epoch & Random variance & Optimized variance & Reduction (\%) \\
\midrule
2  & {\color{magenta}0.20397 $\pm$ 0.00901} & {\color{magenta}0.17367 $\pm$ 0.00751} & 14.85 \\
9  & {\color{magenta}0.23544} $\pm$ 0.00827 & {\color{magenta}0.20531} $\pm$ 0.00644 & 12.80 \\
16 & {\color{magenta}0.24042 $\pm$ 0.00839} & {\color{magenta}0.21829 $\pm$ 0.00651} & 9.20 \\
23 & {\color{magenta}0.10832 $\pm$ 0.00475} & {\color{magenta}0.10365 $\pm$ 0.00407} & 4.32 \\
30 & {\color{magenta}0.00800 $\pm$ 0.00018} & {\color{magenta}0.00794 $\pm$ 0.00017} & 0.78 \\
\bottomrule
\end{tabular*}
\end{table}

\section{Limitations and Future Work}
\label{sec:limitations}

The linear objective of Lemma~\ref{lem:unbiased-selection-polytope} has one variable per subset, so Algorithm~\ref{alg:partition-optimization} settles for a local search over uniform $k$-selections. It accepts only improving swaps, {\color{magenta}uses a fixed iteration budget with no optimality guarantee}, and precomputes all pairwise scores at quadratic cost. Faster or more accurate optimizers, of which the fast variant is a first step, are directions for future work.

{\color{magenta}The inverse-distance surrogate is motivated by the envelopes of the exact kernels in Theorems}~\ref{theorem:torus} {\color{magenta}and}~\ref{theorem:sphere}{\color{magenta}, but decreasing it does not guarantee a reduction in the true variance.} When a single spectral component dominates, the sign of its geometric factor matters: the Bessel factor alternates sign with separation, and even-degree sphere factors favor perpendicular rather than antipodal pairs, so pushing pairs apart can increase the variance. Estimating the dominant components of $f$ from a few labels, and using the exact kernels, could close this gap.

When the data lie on a submanifold of the domain, as images do in an embedding space, most motions carry them off it. The orientation average then evaluates the downstream function, such as the loss of an image classifier, at points that are not images, and the closed forms are no longer informative about the actual variance. To address this, we could use a submersion that maps the data onto a lower-dimensional sphere or torus that they fill, so that no motion carries them off it.

Unbiasedness is a design choice. For particular functions, a biased selection can have lower estimation error than any unbiased one, for instance a smooth $f$ estimated from the evenly spread points of $k$-center greedy. The price is a selected subset whose distribution differs from that of the data, as in the $k$-center greedy column of Figure~\ref{fig:overview} (right). This is {\color{magenta}undesirable} when the subset must represent the population, for instance under class imbalance.

\paragraph{Future work.}
Beyond the directions above, the Haar average extends to other homogeneous spaces, such as hyperbolic and projective embeddings and products of spheres, where harmonic analysis yields analogous pairwise kernels. Weighted selections enlarge the polytope and connect the framework to sensitivity-based coresets. On the application side, the annotation-budget setting that motivates label-free selection, such as remote sensing, remains to be tested.

\section{Conclusion}
\label{sec:conclusion}

Label-free, training-free dataset pruning has relied on the heuristic that nearby points are redundant, which biases the selected subset. This paper replaces the heuristic with a principle. If no orientation of the downstream function relative to the data is preferred, the objective is the sampling variance averaged over all orientations, and the search space is the polytope of unbiased selection algorithms. On the torus and the sphere, the averaged variance has closed pairwise forms in which the unknown function enters only through spectral weights, and the envelopes of the geometric factors give a surrogate that a vertex walk optimizes without labels or training. The resulting selections keep the unbiasedness of uniform sampling and match or exceed its mean accuracy at every budget on CIFAR-10, MNIST, CelebA, and IMDb with the largest gains at small budgets. The same partitions lower gradient variance in mini-batch construction at fixed batch size.

\subsection*{AI use statement}

In this work, we did not use generative AI tools for any task with required disclosure: generating synthetic datasets, developing theoretical models or conceptual frameworks, formulating mathematical claims, providing critical ingredients for proofs, writing proofs, proposing or refining hypotheses, designing or giving feedback on the research methodology or experiments, implementing methods, translation, cleaning or reformatting datasets, qualitative or thematic data analysis, and interpreting results.

We used generative AI tools for the following tasks with recommended disclosure. In writing, they corrected grammar and spelling and gave feedback on the quality of the text. They helped identify relevant literature. They helped create and format the figures and plots. Most of the code in the project was written by hand; generative AI assisted in optimizing CUDA kernels and in orchestrating (but not design) tests. 

We have reviewed all AI-assisted work. Every AI-suggested text edit was reviewed by the authors before being accepted; every AI-suggested reference was read and verified against its original source before citation; and the AI-assisted CUDA kernels and test orchestration were verified and tested for correctness by the authors. We take responsibility for the final content of this work, including text, claims, or artifacts produced with the aid of generative AI.

\subsection*{Reproducibility statement}

All theoretical results are stated with their assumptions and proved in full in the Appendix: Theorems~\ref{theorem:torus} and~\ref{theorem:sphere} give the closed-form variance differences on the torus and the sphere, and Lemma~\ref{lem:unbiased-selection-polytope} establishes the polytope and linear-objective structure the method relies on. The definitions of unbiased selection and of $k$-selections are in Section~\ref{sec:formulation}, and Algorithm~\ref{alg:partition-optimization} specifies the selection procedure completely, with the surrogate score used in all experiments given in Section~\ref{sec:geometry}. The experimental setup in Section~\ref{sec:experiments} lists the datasets, selection fractions, baselines, training framework and hyperparameters, numbers of runs and seeds, and the reporting conventions for accuracies and confidence intervals. {\color{magenta}The remaining dataset-pruning settings and the CelebA autoencoder details are given in Appendices}~\ref{app:pruning-details} {\color{magenta}and}~\ref{app:celeba-vae}, and the mini-batch experiments are described in Section~\ref{sec:minibatches}.

\ificlrfinal
\subsubsection*{Acknowledgments}
The first author is supported by a scholarship from the Brazilian National Council for Scientific and Technological Development (CNPq). The second author is supported by the Institute for Pure and Applied Mathematics (IMPA).
\fi

\bibliography{iclr2027_conference}
\bibliographystyle{iclr2027_conference}

\appendix
\section{Appendix}

\subsection{Theorems}

\begin{theorem}[Torus]
\label{theorem:torus}

Let $\mathcal{T} \subset \R^d / \Z^d = \T^d$ be a finite set of independent
samples from $\mu$, and let $f \in L^2(\T^d \to \R)$.
Identify torus points
with their canonical representatives in $[0,1)^d$, and let $w:\R^d\to[0,1)^d$
denote coordinatewise wrapping modulo one. For a translation $t\in\T^d$ and a
rotation $R\in SO(d)$, define
\[
    g_{t,R}(x) := w(Rx+t).
\]
Thus, rotation of the fixed representatives is followed by translation and
periodic wrapping. The distances below are Euclidean distances between these
representatives. Write $g=g_{t,R}$. For an unbiased selection algorithm
$p\in\mathcal{P}$ and a nonempty subset $S\subseteq\mathcal{T}$, define
\begin{align*}
    V(g, S)
:={}&
    \bigg[
        \int_{\T^d} (f \circ g)(x) \, \dd\mu(x)
    -
        \frac{1}{\abs{S}} \sum_{x \in S} (f \circ g)(x)
    \bigg]^2, \text{ and}
\end{align*}
\[
    V_p(g)
    := \E_{S\sim p}[V(g,S)]
    = \sum_{\emptyset\ne S\subseteq\mathcal{T}} p_S V(g,S).
\]
it follows that $V(I,S)$ is the realized sampled variance of
$S\subseteq\mathcal{T}$, and thus $V_p(I)$ is the mean of the realized sampled
variances of the unbiased selection algorithm $p$ (see
Definition~\ref{def:unbiased}).

If $p,q\in\mathcal{P}$ are unbiased selection algorithms and $\mathcal{G}$ is
the distribution of $g_{t,R}$ induced by independently sampling $t$ and $R$
from the normalized Haar measures on $\T^d$ and $SO(d)$, respectively, then

\begin{align*}
    \E_{g \sim \mathcal{G}}
    \left[
        V_p(g)
    -
        V_q(g)
    \right]
&=
    \E_{S \sim p} \left[
        \frac{1}{\abs{S}^2}
        \sum_{x,y \in S}
        v_{\T^d}\big(
            \norm{x - y}_2
        \big)
    \right] \\
&\qquad-
    \E_{S \sim q} \left[
        \frac{1}{\abs{S}^2}
        \sum_{x,y \in S}
        v_{\T^d}\big(
            \norm{x - y}_2
        \big)
    \right]
\end{align*}
in which
$$
    v_{\T^d}(r)
=
    \sum_{j \in \Z^d}
    \abs{\hat{f}(j)}^2
    \frac{\Gamma(d/2)}
         {\left(\pi \norm{j}_2 r\right)^{d/2 -1}}
    J_{d/2-1} (2 \pi \norm{j}_2 r)
$$
and $J_\alpha$ is the  Bessel function of the first kind.
The normalized Bessel factor is defined by its continuous extension,
with value one when $\norm{j}_2r=0$.
\end{theorem}

\begin{proof}
Expanding the square yields
\begin{align}
    \E_{g \sim \mathcal{G}} \big[
        V(g, S)
    \big]
= &
    \E_{g \sim \mathcal{G}} \left[
        \left[
            \int_{\T^d} (f \circ g)(x) \, \dd\mu(x)
        -
            \frac{1}{\abs{S}}
            \sum_{a \in S} (f \circ g)(a)
        \right]^2
    \right] \\
= &
    \label{eq:first_vga_torus}
    \E_{g \sim \mathcal{G}} \left[
        c_S(g)
    \right]
    +
        \frac{1}{\abs{S}^2}
        \E_{g \sim \mathcal{G}} \left[
            \left[
                \sum_{a \in S} (f \circ g)(a)
            \right]^2
        \right].
\end{align}

Note that for almost every $g$,
\begin{align*}
    \E_{S \sim p}[ c_S(g) ]
=&
    \left(\int_{\T^d} (f \circ g)(x) \, \dd\mu(x)\right)^2 \\
&-
    2\int_{\T^d} (f \circ g)(x) \, \dd\mu(x)
    \sum_{\emptyset\ne S\subseteq\mathcal{T}} \frac{p_S}{\abs{S}}
    \sum_{a \in S} (f \circ g) (a)
\end{align*}

It follows from our definition of unbiased selection algorithm that

\begin{align*}
    \E_{S \sim p}[ c_S(g) ]
=&
    \left(\int_{\T^d} (f \circ g)(x) \, \dd\mu(x)\right)^2 \\
&-
    2\int_{\T^d} (f \circ g)(x) \, \dd\mu(x)
    \frac{1}{\abs{\mathcal{T}}} \sum_{z \in \mathcal{T}}(f \circ g) (z),
\end{align*}

which is independent on the algorithm $p$ and thus

$$
    \E_{S \sim p}[ c_S(g) ]
=
    \E_{S \sim q}[ c_S(g) ].
$$

One can conclude that both terms cancel out and that we can concentrate on the second element of Equation \ref{eq:first_vga_torus}.
Let us define $\mathcal{R}$ as the uniform distribution over rotation matrices in $SO(d)$.
Write $\widetilde f=f\circ w$ for the periodic extension of $f$.
Then, it follows that

\begin{align*}
    \frac{1}{\abs{S}^2}
    \E_{g \sim \mathcal{G}} \left[
        \left[
             \sum_{a \in S} (f \circ g)(a)
        \right]^2
    \right]
& =
    \frac{1}{\abs{S}^2}
    \sum_{a,b \in S}
    \E_{g \sim \mathcal{G}} \left[
        (f \circ g)(a)(f \circ g)(b)
    \right]\text{,}
\end{align*}
 in which
\begin{align*}
    \E_{g \sim \mathcal{G}} \left[
        (f \circ g)(a)(f \circ g)(b)
    \right]
&=
    \E_{R \sim \mathcal{R}} \left[
        \int_{[0,1)^d}
            \widetilde f(Ra+t)\widetilde f(Rb+t)
        \,\dd t
    \right].
\end{align*}

Using the substitution $y = Ra+t$, it follows that $\dd y = \dd t$.
The translated domain $Ra+[0,1)^d$ can be replaced by $[0,1)^d$ because
$\widetilde f(y)\widetilde f(R(b-a)+y)$ is $\Z^d$-periodic, and we can proceed by applying Parseval's theorem

\begin{align*}
    \E_{R \sim \mathcal{R}} \left[
        \int_{[0,1)^d}
            \widetilde f(Ra+t)\widetilde f(Rb+t)
        \,\dd t
    \right]
& =
    \E_{R \sim \mathcal{R}} \left[
        \int_{[0,1)^d}
            \widetilde f(y)\widetilde f(R(b-a)+y)
        \,\dd y
    \right] \\
& =
        \E_{R \sim \mathcal{R}} \left[
            \sum_{j \in \Z^d}
            \F[f](j) \overline {\F[\widetilde f(\cdot+R(b-a))](j)}
        \right]\\
& =
        \E_{R \sim \mathcal{R}} \left[
            \sum_{j \in \Z^d}
            \abs{\hat{f}(j)}^2 e^{-2 \pi i j \cdot R(b - a)}
        \right]\\
& =
        \sum_{j \in \Z^d}
            \abs{\hat{f}(j)}^2
        \E_{R \sim \mathcal{R}} \left[
             e^{-2 \pi i j \cdot R(b - a)}
        \right].
\end{align*}
Here $\hat f(j)=\int_{[0,1)^d}f(y)e^{-2\pi i j\cdot y}\,\dd y$;
the interchange of expectation and summation follows from
$\sum_j\abs{\hat f(j)}^2<\infty$.

For $d=1$, the formula simplifies to

\begin{align*}
    \sum_{j \in \Z}
        \abs{\hat{f}(j)}^2
        \E_{R\sim\mathcal R}\left[e^{-2\pi i j R(b-a)}\right]
= &
    \sum_{j \in \Z}
        \abs{\hat{f}(j)}^2
         \cos\left(2 \pi j \norm{b-a}_2\right).
\end{align*}
Indeed, $SO(1)=\{1\}$ and $\abs{\hat f(-j)}^2=\abs{\hat f(j)}^2$ since $f$ is real.

Otherwise, $\abs{S^{d-2}}$ is defined and it follows that
(with $\alpha=d/2-1$)

\begin{align*}
    \E_{R\sim\mathcal R}\left[e^{-2\pi i j\cdot R(b-a)}\right]
=&
    \frac{1}{\abs{S^{d-1}}}
    \int_{S^{d-1}}
         e^{-2 \pi i j \cdot x \norm{b-a}_2}
         \dd x \\
=&
    \frac{1}{\abs{S^{d-1}}}
    \int_{0}^\pi
         e^{-2 \pi i \norm{b-a}_2 \norm{j}_2 \cos \theta}
         \abs{S^{d-2}} \sin(\theta)^{d-2}
         \dd \theta \\
=&
    \frac{\frac{2 \pi^{(d-1)/2}}{\Gamma((d-1)/2)}}{\frac{2 \pi^{d/2}}{\Gamma(d/2)}}
    \int_{0}^\pi
         e^{-2 \pi i \norm{b-a}_2 \norm{j}_2 \cos \theta}
         \sin(\theta)^{d-2}
         \dd \theta \\
=&
    \frac{\Gamma(d/2)}{\Gamma((d-1)/2) \sqrt{\pi}}
    \int_{0}^\pi
         e^{-2 \pi i \norm{b-a}_2 \norm{j}_2 \cos \theta}
         \sin(\theta)^{d-2}
         \dd \theta \\
=&
    \frac{\Gamma(\alpha + 1)}{\Gamma(\alpha + 1/2) \sqrt{\pi}}
    \int_{0}^\pi
         e^{-2 \pi i \norm{b-a}_2 \norm{j}_2 \cos \theta}
         \sin(\theta)^{2 \alpha}
         \dd \theta \\
=&
    \left.
        \frac{\Gamma(d/2)}{\left(\frac{t}{2}\right)^{d/2 -1}} J_{d/2-1} (t)
    \right|_{t = 2 \pi \norm{j}_2 \norm{b-a}_2}.
\end{align*}


\end{proof}

\begin{theorem}[Sphere]
\label{theorem:sphere}
Let $\mathcal{T} \subset S^d$ be a finite set of independent
samples from $\mu$, $g : S^d \to S^d$ a rigid motion
and $f \in L^2(S^d \to \R)$. Let $\sigma$ be the normalized surface measure
on $S^d$, with $\sigma(S^d)=1$, and let $f_\ell$ be the degree-$\ell$ spherical
harmonic component of $f$. Its squared norm is
\[
    \norm{f_\ell}_{L^2(\sigma)}^2
    :=\int_{S^d}\abs{f_\ell(x)}^2\,\dd\sigma(x).
\]
For an unbiased selection algorithm
$p\in\mathcal{P}$ and a nonempty subset $S\subseteq\mathcal{T}$, define
\begin{align*}
    V(g, S)
:={}&
    \bigg[
        \int_{S^d} (f \circ g)(x) \, \dd\mu(x)
    -
        \frac{1}{\abs{S}} \sum_{x \in S} (f \circ g)(x)
    \bigg]^2, \text{ and}
\end{align*}
\[
    V_p(g)
    := \E_{S\sim p}[V(g,S)]
    = \sum_{\emptyset\ne S\subseteq\mathcal{T}} p_S V(g,S).
\]
it follows that $V(I,S)$ is the realized sampled variance of
$S\subseteq\mathcal{T}$, and thus $V_p(I)$ is the mean of the realized
sampled variances of the unbiased selection algorithm $p$ (see
Definition~\ref{def:unbiased}).

If $p,q\in\mathcal{P}$ are unbiased selection algorithms and
$\mathcal{G}$ is the uniform distribution over the isometries in $S^d$, then
\begin{align*}
    \E_{g \sim \mathcal{G}}
    \left[
        V_p(g)
    -
        V_q(g)
    \right] 
&=
    \E_{S \sim p} \left[
        \frac{1}{\abs{S}^2}
        \sum_{x,y \in S}
        v_{S^d}\big(
            x \cdot y
        \big)
    \right] \\
&\qquad-
    \E_{S \sim q} \left[
        \frac{1}{\abs{S}^2}
        \sum_{x,y \in S}
        v_{S^d}\big(
            x \cdot y
        \big)
    \right]
\end{align*}
in which
$$
    v_{S^d}(t)
:=
    \sum_{\ell \in \N}
    \norm{f_\ell}_{L^2(\sigma)}^2
    P_\ell^{d+1}(t),
$$
and $P_\ell^{d+1}$ is the generalized Legendre polynomial of dimension $d+1$ and
degree $\ell$, that can alternatively be defined for $d\ge2$ as:
$$
    P_\ell^{d+1}(t)
=
    \frac{C_\ell^{(\lambda)}(t)}{C_\ell^{(\lambda)}(1)}
$$
in which $C_\ell^{(\lambda)}$ are the Gegenbauer polynomials, and $\lambda = (d-1)/2$.
For $d=1$, we use the limiting definition
\[
    P_\ell^2(t)=\cos(\ell\arccos t),\qquad t\in[-1,1].
\]
\end{theorem}

\begin{proof}
Expanding the square yields
\begin{align}
    \E_{g \sim \mathcal{G}} \big[
        V(g, S)
    \big]
= &
    \E_{g \sim \mathcal{G}} \left[
        \left[
            \int_{S^d} (f \circ g)(x) \, \dd\mu(x)
        -
            \frac{1}{\abs{S}}
            \sum_{a \in S} (f \circ g)(a)
        \right]^2
    \right] \\
= &
    \label{eq:first_vga_sphere}
    \E_{g \sim \mathcal{G}} \left[
        c_S(g)
    \right]
    +
        \frac{1}{\abs{S}^2}
        \E_{g \sim \mathcal{G}} \left[
            \left[
                \sum_{a \in S} (f \circ g)(a)
            \right]^2
        \right].
\end{align}

As shown in Theorem \ref{theorem:torus}, unbiasedness gives
\[
    \E_{S\sim p}[c_S(g)]=\E_{S\sim q}[c_S(g)]
\]
for almost every $g$. Thus these terms
cancel out in $V_p(g)-V_q(g)$, and we can focus on the second element of Equation \ref{eq:first_vga_sphere}.

\begin{align}
    \frac{1}{\abs{S}^2} \E_{g \sim \mathcal{G}} \left[
        \left[
            \sum_{a \in S} (f \circ g)(a)
        \right]^2
    \right]
&=
    \frac{1}{\abs{S}^2}
    \sum_{a, b \in S}
    \E_{Q \sim \mathcal{G}} \left[
        f(Q a)f(Q b)
    \right].
\end{align}

Let $R$ be any rotation such that $R a = a$. Since all isometries on the sphere
are orthogonal transformations, we will denote $g$ with the letter $Q$:

\begin{align*}
\alpha
:=&
    \E_{Q \sim \mathcal{G}} \left[
        f(Q a)f(Q b)
    \right] \\
=&
    \E_{Q \sim \mathcal{G}} \left[
        f(Q a)f(Q R b)
    \right] \\
=&
    \E_{Q \sim \mathcal{G}} \left[
        \frac{1}{\abs{C_{Qa}}}
        \int_{x \in C_{Qa}}
        f(Q a)f(x)
        \dd x
    \right],
\end{align*}

in which $C_y = \{\abs{y} x: x \in S^d \land x \cdot y/\abs{y} = a \cdot b\}$.
At $a\cdot b=1$, each normalized slice integral below is interpreted
as evaluation at $y$; at $a\cdot b=-1$, it is interpreted as evaluation at $-y$.
Note that now the expected value only depends on $Qa$ rather than $Q$, and it follows that

\begin{align*}
    \alpha
=&
    \frac{1}{\abs{S^d}}
    \int_{y \in S^d}
    \frac{1}{\abs{C_y}}
    \int_{x \in C_y}
    f(y)f(x)
    \dd x
    \dd y \\
=&
    \frac{1}{\abs{S^d}}
    \int_{y \in S^d}
    f(y)
    \frac{1}{\abs{C_y}}
    \int_{x \in C_y}
    f(x)
    \dd x
    \dd y \\
=&
    \frac{1}{\abs{S^d}}
    \int_{y \in S^d}
    f(y)
    \overline{h(y)}
    \dd y \\
=&
    \langle f,h\rangle_{L^2(\sigma)}.
\end{align*}

Let $A_\ell$ denote the space of degree-$\ell$ spherical harmonics on
$S^d$, identified with their homogeneous harmonic extensions to $\R^{d+1}$.

By expressing $f$ as the direct sum
$$
    f = \widehat\bigoplus_{\ell \in \N} f_\ell, f_\ell \in A_\ell,
$$

it follows that 

\begin{align}
\label{eq:expand_g}
    h(y)
=&
    \frac{1}{\abs{C_y}}
    \int_{x \in C_y}
    f(x)
    \dd x \\
=&
    \frac{1}{\abs{C_y}}
    \int_{x \in C_y}
    \sum_{\ell \in \N} f_\ell(x)
    \dd x \\
=&
    \sum_{\ell \in \N}
    \frac{1}{\abs{C_y}}
    \int_{x \in C_y}
    f_\ell(x)
    \dd x.
\end{align}

We want to show that

\begin{align*}
    h_\ell(y)
:=&
    \frac{1}{\abs{C_y}}
    \int_{x \in C_y}
    f_\ell(x)
    \dd x
\end{align*}

is in $A_\ell$.

\begin{enumerate}
\item $h_\ell(\lambda y) = \lambda^\ell h_\ell(y)$

For $\lambda\ne0$,
\begin{align*}
    h_\ell(\lambda y)
= &
    \frac{1}{\abs{C_{\lambda y}}}
    \int_{x \in C_{\lambda y}}
    f_\ell(x)
    \dd x \\
= &
    \frac{1}{\abs{\lambda}^{d-1} \abs{C_y}}
    \int_{x \in C_{y}}
    f_\ell(\lambda x)
    \abs{\lambda}^{d-1} \dd x \\
= &
    \lambda^\ell h_\ell(y).
\end{align*}

\item $\Delta h_\ell = 0$

This is a consequence of the Funk--Hecke theorem \cite{groemer}, since $\forall y \in S^d$,
$h_\ell(y)$ $={}$ $\beta_\ell(a \cdot b)$ $f_\ell(y)$ for some function $\beta_\ell$$: [-1,1]
\to \R$ that will be characterized latter. The previous item can be used to
extend the equality to $\R^{d+1}$, thus proving that $h_\ell(y)$ $\in A_\ell$.

\end{enumerate}

$$
    \therefore h = \widehat\bigoplus_{\ell \in \N} h_\ell, h_\ell \in A_\ell
$$

Then, we can use the orthogonality of the spherical harmonics and the Funk--Hecke theorem to conclude that

\begin{align*}
    \alpha
=&
    \sum_{\ell \in \N} \langle f_\ell,h_\ell\rangle_{L^2(\sigma)} \\
=&
    \sum_{\ell \in \N} \norm{f_\ell}_{L^2(\sigma)}^2\,\beta_\ell(a \cdot b) \\
=&
    \sum_{\ell \in \N}
    \norm{f_\ell}_{L^2(\sigma)}^2
    P_\ell^{d+1}(a \cdot b),
\end{align*}
in which $P_\ell^{d+1}$ is the generalized Legendre polynomial of dimension $d+1$ and degree $\ell$.
\end{proof}

\FloatBarrier
{\color{magenta}A selection algorithm $p\in\mathcal{P}$ is unbiased in the sense of Definition}~\ref{def:unbiased} {\color{magenta}if and only if
\[
    p_{\emptyset}=0,
    \qquad
    \sum_{\substack{S\subseteq\mathcal{T}\\x\in S}}
        \frac{p_S}{\abs{S}}
    =\frac{1}{\abs{\mathcal{T}}}
    \quad\text{for every }x\in\mathcal{T}.
\]
Indeed, the empty subset must have zero probability for the subset average to be defined. Expanding the expectation over nonempty subsets gives
\[
    \E_{S\sim p}\left[\frac{1}{\abs{S}}\sum_{x\in S}f(x)\right]
    =\sum_{x\in\mathcal{T}}f(x)
      \sum_{\substack{S\subseteq\mathcal{T}\\x\in S}}
          \frac{p_S}{\abs{S}}.
\]
The stated conditions therefore imply unbiasedness; conversely, applying unbiasedness to the indicator function of each individual point gives these conditions. For subsets of fixed size $k\geq1$, they reduce to equal inclusion probabilities:
\[
    \Pr_{S\sim p^{(k)}}(x\in S)
    =\sum_{\substack{S\subseteq\mathcal{T}\\x\in S}}p^{(k)}_S
    =\frac{k}{\abs{\mathcal{T}}}
    \quad\text{for every }x\in\mathcal{T}.
\]
These linear constraints yield the polytope characterization below.}

\begin{lemma}
\label{lem:unbiased-selection-polytope}

Let $\binom{\mathcal{T}}{k} = \left\{S_1, \cdots,
S_{\binom{\abs{\mathcal{T}}}{k}}\right\}$ be an enumeration of the subsets of
$\mathcal{T}$ with size $k$. Note that any unbiased selection algorithm of size
$k$ can be uniquely identified as

$$
    p^{(k)} \in [0,1]^{\binom{\abs{\mathcal{T}}}{k}}\text{, in which}
$$
$$
    p_i^{(k)} = p^{(k)}_{S_i}.
$$

Similarly, any loss function $\mathcal{L} : \binom{\mathcal{T}}{k} \to \R$ can be uniquely identified as

$$
    \mathcal{L} \in \R^{\binom{\abs{\mathcal{T}}}{k}}\text{, in which}
$$
$$
    \mathcal{L}_i = \mathcal{L}(S_i).
$$

It follows from the definition of expectation that

$$
    \E_{S \sim p^{(k)}}[\mathcal{L}(S)]
=
    \left\langle p^{(k)},\mathcal{L}\right\rangle_{\R^{\binom{\abs{\mathcal{T}}}{k}}}
=
    p^{(k)}\cdot\mathcal{L}.
$$

Thus,

\begin{equation}
\label{eq:linear_programming}
    p^*
=
    \argmin_{p^{(k)} \in \mathcal{U}_k}
    \E_{S \sim p^{(k)}}[\mathcal{L}(S)],
\end{equation}

is a linear programming problem.
If $k$ divides $\abs{\mathcal{T}}$ and $p^{(k)}$ is a uniform $k$-selection, then $p^{(k)}$ is a vertex of the convex polytope $\mathcal{U}_k$. 
\end{lemma}

\begin{proof}
{\color{magenta}Enumerate $\mathcal{T}=\{t_1,\ldots,t_{\abs{\mathcal{T}}}\}$.} Let
$$
{\color{magenta}B}
    :=
{\color{magenta}\begin{bmatrix}
    c(S_1)     & c(S_2)      & \cdots    & c\left(S_{\binom{\abs{\mathcal{T}}}{k}}\right)
\end{bmatrix}
\in \R^{\abs{\mathcal{T}} \times \binom{\abs{\mathcal{T}}}{k}}},
$$

in which ${\color{magenta}c(S_j)}$ are the binary column vectors

$$
    {\color{magenta}c(S_j)_i} :=
{\color{magenta}\begin{cases}
    1 \text{ if } t_i \in S_j \\
    0 \text{ if } t_i \not\in S_j
\end{cases}}.
$$

Using this definition, we can rewrite Equation~\ref{eq:linear_programming} as the following optimization problem:

Given 
$
    \mathcal{L} \in \R^{
    \binom{
        \abs{
            \mathcal{T}
            }
    }
    {k}}
$,
minimize

$$
    {\color{magenta}p^{(k)}} \cdot \mathcal{L},
$$

in which ${\color{magenta}p^{(k)}} \in \R^{\binom{\abs{\mathcal{T}}}{k}}$ is subject to being a probability measure on $\binom{\mathcal{T}}{k}$:

\begin{align}
    \label{eq:lemma_nonneg}
    \forall 1\leq i  \leq \binom{\abs{\mathcal{T}}}{k},\,
        0 \leq {\color{magenta}p^{(k)}_i = p^{(k)}_{S_i}}\\
    \sum_{i=1}^{\binom{\abs{\mathcal{T}}}{k}} {\color{magenta}p_i^{(k)}} = 1
\end{align}

and being unbiased, i.e.
\begin{align}
    \label{eq:lemma_unbiased}
    \forall t \in \mathcal{T},\, {\color{magenta}\sum_{\substack{1\leq i\leq\binom{\abs{\mathcal{T}}}{k} \\ t \in S_i}} \frac{p_i^{(k)}}{k}} = \frac{1}{\abs{\mathcal{T}}}.
\end{align}

From Equation \ref{eq:lemma_unbiased} and non-negativity, it follows immediately that

$$
    \forall i, \, {\color{magenta}p^{(k)}_i}  \leq \frac{k}{\abs{\mathcal{T}}}.
$$

Thus, we can rewrite Equations \ref{eq:lemma_nonneg}-\ref{eq:lemma_unbiased} as follows:

\begin{align}
    0 \leq {\color{magenta}p^{(k)}_i}  \leq \frac{k}{\abs{\mathcal{T}}} \\
    {\color{magenta}\mathbf{1}^{\mathsf T} p^{(k)}} = 1 \\
    {\color{magenta}B p^{(k)}} = \frac{k}{\abs{\mathcal{T}}} \mathbf{1}
\end{align}

which indeed is a convex polytope ${\color{magenta}\mathcal{U}_k} \subset
\R^{\binom{\abs{\mathcal{T}}}{k}}$. Since in real-world applications the number
of dimensions would surpass the number of atoms in the universe, traditional
linear programming techniques are not applicable. However, we can find vertices
of ${\color{magenta}\mathcal{U}_k}$ by solving the surrogate problem below.

{\color{magenta}For the vertex claim, assume $k\mid\abs{\mathcal{T}}$.} For $v \in \R^{\binom{\abs{\mathcal{T}}}{k}}$, maximize

$$
    {\color{magenta}\lVert v\rVert_2^2}
$$

Subject to $v \in \mathcal{V}$, a convex polytope defined by the following two equations
\begin{align}
    \label{eq:restrictions_simplified_1}
    0 \leq v_i  \leq \frac{k}{\abs{\mathcal{T}}} \\
    \label{eq:restrictions_simplified_2}
    {\color{magenta}\mathbf{1}^{\mathsf T}} v = 1
\end{align}

The problem has at least one solution, since $\mathcal{V} \neq \emptyset$ is a compact set and
${\color{magenta}\lVert v\rVert_2^2}$ is continuous. We will show that if $v^*$ is a solution to the
problem, then $v_i^* \in \left\{0, \frac{k}{\abs{\mathcal{T}}} \right\}$.
Suppose by hypothesis that $u$ maximizes ${\color{magenta}\lVert u\rVert_2^2}$, but

$$
    \exists i_0, 0 < u_{i_0} < \frac{k}{\abs{\mathcal{T}}}.
$$

Since $k \mid \abs{\mathcal{T}}$ and ${\color{magenta}\mathbf{1}^{\mathsf T}}u = 1$, it follows that

$$
    \exists i_1 \neq i_0, 0 < u_{i_1} < \frac{k}{\abs{\mathcal{T}}}.
$$

This implies that $u$ is not a solution, since

$$
    u'_i
=
    \begin{cases}
        u_i &\text{ if } i\neq i_0 \land i \neq i_1 \\
        \min\left\{\frac{k}{\abs{\mathcal{T}}}, u_{i_0} + u_{i_1}\right\} &\text{ if } i = i_0 \\
        \max\left\{u_{i_0} + u_{i_1} - \frac{k}{\abs{\mathcal{T}}}, 0\right\} &\text{ if } i = i_1
    \end{cases}
$$

satisfies the Equations \ref{eq:restrictions_simplified_1} and \ref{eq:restrictions_simplified_2}, and ${\color{magenta}\lVert u'\rVert_2^2 > \lVert u\rVert_2^2}$. Thus, there are $\binom{\binom{\abs{\mathcal{T}}}{k}}{\frac{\abs{\mathcal{T}}}{k}}$ solutions $v^* \in \mathcal{V}$ that maximize the distance to the origin, which includes all {\color{magenta}uniform $k$-selections} (see Definition \ref{def:permutation}). Note that among these solutions, only the {\color{magenta}uniform $k$-selections} belong to ${\color{magenta}\mathcal{U}_k}$.

{\color{magenta}Since squared Euclidean norm is strictly convex, these maximizers are vertices of $\mathcal{U}_k$.}
\end{proof}

\section{Additional experimental details}
\label{app:experimental-details}

\subsection{Dataset-pruning settings}
\label{app:pruning-details}

{\color{magenta}\paragraph{CIFAR-10 and MNIST.}
We use DeepCore} \citep{guo2022} {\color{magenta}with ResNet18 and the standard training and test splits. All results in Table}~\ref{tab:cifar-mnist} {\color{magenta}use selection without class-balanced quotas. For the embedding-based methods, a randomly initialized ResNet18 is trained on the full labeled training set for 10 epochs with cross-entropy loss. Selection uses its penultimate-layer features, extracted in evaluation mode after the final epoch. Thus, the selection rules require neither labels nor additional training given the embeddings, but producing these embeddings involves supervised training. Evaluation uses a separate, freshly initialized ResNet18 trained on the selected subset for 200 epochs.}

{\color{magenta}\paragraph{Preprocessing and optimization.}
CIFAR-10 images are normalized channelwise with means $(0.4914,0.4822,0.4465)$ and standard deviations $(0.2470,0.2435,0.2616)$; MNIST uses mean $0.1307$ and standard deviation $0.3081$. Downstream CIFAR-10 training additionally uses random crops with four-pixel reflection padding and random horizontal flips. Embedding training, feature extraction, and test evaluation use no augmentation; MNIST uses no augmentation in either training stage. Both training stages use SGD with batch size 128, momentum $0.9$, Nesterov acceleration, weight decay $5\times10^{-4}$, and cross-entropy loss. Embedding training uses a constant learning rate of $0.1$; downstream training uses an initial learning rate of $0.1$ with cosine annealing to $10^{-4}$ over 200 epochs, updated after each mini-batch.}

\subsection{CelebA experimental details}
\label{app:celeba-vae}

{\color{magenta}\paragraph{VAE embeddings.}
We use an ADM-style residual and self-attention VAE pretrained on aligned CelebA images. RGB images are converted to floating-point values in $[0,1]$ and resized bilinearly to $128\times128$. The model has approximately 130 million parameters and a 512-dimensional Gaussian latent space. Its residual blocks use GroupNorm and SiLU, with four-head self-attention at spatial resolutions $32\times32$ and below; the decoder ends with a sigmoid. The VAE is trained for 140 epochs with batch size 32, learning rate $10^{-4}$, and zero weight decay. The objective combines a Gaussian reconstruction likelihood with standard deviation $0.1$, KL regularization with a free-bits threshold of $0.5$, and a VGG perceptual loss weighted by $0.5$. Selection and downstream classification use the deterministic 512-dimensional posterior means, extracted with the VAE in evaluation mode, rather than sampled latent vectors.}

{\color{magenta}\paragraph{Downstream classification.}
The selection pool contains 40.5k embedding vectors; selection fractions are relative to this pool, and the selected vectors are used for downstream training. Evaluation uses a separate 20.2k-vector test split. The classifier is an MLP with widths $512\to256\to128\to20$. Each hidden linear layer is followed by LayerNorm, ReLU, and dropout with probability $0.2$; the output layer produces 20 logits. We train for 20 epochs with Adam, a constant learning rate of $10^{-3}$, batch size 256, and binary cross-entropy with logits.}

{\color{magenta}The targets are Male, Young, Attractive, Smiling, Heavy Makeup, Wearing Lipstick, Blond Hair, Black Hair, Brown Hair, Gray Hair, Bald, Bangs, Eyeglasses, Goatee, Mustache, No Beard, Wearing Hat, Wearing Earrings, Pale Skin, and Chubby. Labels are mapped from $\{-1,1\}$ to $\{0,1\}$, and predictions use a sigmoid threshold of $0.5$. Mean test attribute accuracy averages the per-attribute test accuracies over these 20 targets. Each run contributes its highest mean test attribute accuracy over the 20 training epochs, and the table averages these values over 100 runs per setting.}

\subsection[IMDb experimental details]{{\color{magenta}IMDb experimental details}}
\label{app:imdb-details}

{\color{magenta}\paragraph{Data and embeddings.}
We use the standard labeled IMDb splits: 25,000 training reviews and 25,000 test reviews, each containing equal numbers of positive and negative examples. The unlabeled split is not used. Reviews are embedded with the pretrained Qwen3-Embedding-8B model} \citep{zhang2025}{\color{magenta}, using the original review text without an added task-specific instruction and a truncation limit of 8,192 tokens. The resulting 4,096-dimensional, unit-$\ell_2$-norm embeddings are stored in float32 and used unchanged for selection and downstream classification, with no additional normalization, dimensionality reduction, or augmentation. Selection operates only on the training split, without class-balanced quotas or additional embedding training.}

{\color{magenta}\paragraph{Classifier and optimization.}
For each selected subset, we train a freshly initialized MLP with widths $4096\to128\to128\to2$, ReLU activations after both hidden layers, and no dropout or normalization layers. The output consists of two sentiment logits, trained with cross-entropy loss. Training runs for 200 epochs using SGD with batch size 128, momentum $0.9$, Nesterov acceleration, and weight decay $5\times10^{-4}$. The learning rate starts at $0.1$ and follows cosine annealing to $10^{-4}$ over 200 epochs, updated after each mini-batch. Training examples are shuffled each epoch.}

{\color{magenta}\paragraph{Epoch selection and reporting.}
We evaluate accuracy on the full test split after every epoch. Each run contributes its highest test accuracy over the 200 training epochs; no separate validation split is used for epoch selection. The main IMDb comparison uses 200 selection-and-training runs per method--fraction setting. Table}~\ref{tab:imdb} {\color{magenta}reports their mean accuracy and 95\% Wald confidence-interval half-width, computed as $1.96s/\sqrt{200}$, where $s$ is the sample standard deviation of the per-run best-epoch accuracies.}

\end{document}

%% file: math_commands.tex
\usepackage{amsmath,amsfonts,bm}

\def\eqref#1{equation~\ref{#1}}

\def\1{\bm{1}}

\DeclareMathAlphabet{\mathsfit}{\encodingdefault}{\sfdefault}{m}{sl}
\SetMathAlphabet{\mathsfit}{bold}{\encodingdefault}{\sfdefault}{bx}{n}

\newcommand{\E}{\mathbb{E}}

\newcommand{\R}{\mathbb{R}}

\DeclareMathOperator*{\argmin}{arg\,min}

%% file: header.tex
\usepackage{xparse}

\definecolor{FranciscoComment}{RGB}{0,82,155}
\definecolor{RodrigoComment}{RGB}{180,0,120}

\newif\ifshowcomments
\showcommentsfalse

\NewDocumentCommand{\comment}{m m +m +m}{%
    \ifshowcomments
        {\color{#2}#3%
        \if\relax\detokenize{#4}\relax
        \else
            \space\textbf{[#1: #4]}%
        \fi}%
    \else
        #3%
    \fi
}

\NewDocumentCommand{\namedcomment}{m m +m o}{%
    \IfNoValueTF{#4}
        {\comment{#1}{#2}{#3}{}}
        {\comment{#1}{#2}{#3}{#4}}%
}

\NewDocumentCommand{\francisco}{+m o}{%
    \IfNoValueTF{#2}
        {\namedcomment{Francisco}{FranciscoComment}{#1}}
        {\namedcomment{Francisco}{FranciscoComment}{#1}[#2]}%
}

\NewDocumentCommand{\rodrigo}{+m o}{%
    \IfNoValueTF{#2}
        {\namedcomment{Rodrigo}{RodrigoComment}{#1}}
        {\namedcomment{Rodrigo}{RodrigoComment}{#1}[#2]}%
}

\usepackage{soul}
\definecolor{EditHighlight}{RGB}{255,244,130}
\newif\ifmarkedits
\markeditsfalse

\AtBeginDocument{\ifmarkedits\else\colorlet{magenta}{black}\fi}

\makeatletter
\newcommand{\EditHighlightLeaders}{%
    \leaders\hbox{{%
        \color{EditHighlight}%
        \vrule width .1pt height \SOUL@ulht depth \SOUL@uldp
    }}%
}
\NewDocumentCommand{\edited}{+m}{%
    \ifmarkedits
        {\let\SOUL@hlcolor\@firstofone
         \let\SOUL@ulleaders\EditHighlightLeaders
         \hl{#1}}%
    \else
        #1%
    \fi
}
\makeatother


%% file: cmd.tex
\newcommand{\enun}[1]%
{%
    \ifthenelse{\isundefined{\showenun}}%
        {}%
        {#1}%
}

\newcommand{\green}[1]{\ifmarkedits\textcolor[rgb]{0,0.5,0}{#1}\else#1\fi}
\newcommand{\purple}[1]{\ifmarkedits\textcolor[rgb]{0.5,0,0.5}{#1}\else#1\fi}

\providecommand{\argmin}{\operatorname*{arg\,min}}

\newcommand{\dd}{\mathop{}\mathrm{d}}
\newcommand{\norm}[1]{\left\lVert#1\right\rVert}
\newcommand{\abs}[1]{\left\lvert#1\right\rvert}

\newcommand{\F}{\mathcal{F}}

\providecommand{\R}{\mathbb{R}}
\newcommand{\N}{\mathbb{N}}
\newcommand{\Z}{\mathbb{Z}}
\newcommand{\T}{\mathbb{T}}

\providecommand{\E}{\mathbb{E}}

\theoremstyle{definition}
\newtheorem{definition}{Definition}

\theoremstyle{plain}
\newtheorem{theorem}{Theorem}[section]
\newtheorem{lemma}[theorem]{Lemma}